\documentclass[letterpaper, 10 pt, conference]{ieeeconf}  % Comment this line out if you need a4paper

\IEEEoverridecommandlockouts                              % This command is only needed if 
\usepackage{amsthm}

\usepackage{amsmath} % assumes amsmath package installed
\usepackage{amssymb}  % assumes amsmath package installed

\usepackage{subcaption}
\usepackage{algorithm}% http://ctan.org/pkg/algorithms
\usepackage{algpseudocode}% http://ctan.org/pkg/algorithmicx

\usepackage{url}  %Required
\usepackage{graphicx}  %Required

\usepackage{hyperref}

\newcommand{\set}[1]{\ensuremath{\mathcal{#1}}}

\newcommand{\pixelval}{\ensuremath{I_{mn}^t}}

\theoremstyle{plain}

\newtheorem{lemma}{Lemma}

\theoremstyle{definition}
\newtheorem{defn}{Definition} % definition numbers are dependent on theorem numbers
\newtheorem{problem}{Problem} % same for example numbers
\title{\LARGE \bf
Image Space Potential Fields: Constant Size Environment Representation for Vision-based Subsumption Control Architectures
}

\author{Jeffrey Kane Johnson$^{1}$% <-this % stops a space
\thanks{$^{1}$Jeffrey Kane Johnson received his PhD from Indiana University and is principal of Maeve Automation,
        Mountain View, CA 94043\newline
        \href{mailto:contact@maeveautomation.com}{\nolinkurl{contact@maeveautomation.com}}}%
}

\usepackage[pscoord]{eso-pic}
\newcommand{\placetextbox}[3]{
\setbox0=\hbox{#3}
\AddToShipoutPictureFG*{ \put(\LenToUnit{#1\paperwidth},\LenToUnit{#2\paperheight}){\vtop{{\null}\makebox[0pt][c]{#3}}}
}
}
\placetextbox{.2}{0.055}{TR-1~\copyright~2017 Maeve Automation}

\begin{document}
\maketitle
%\thispagestyle{empty}
%\pagestyle{empty}

%%%%%%%%%%%%%%%%%%%%%%%%%%%%%%%%%%%%%%%%%%%%%%%%%%%%%%%%%%%%%%%%%%%%%%%%%%%%%%%%
\begin{abstract}
This technical report presents an environment representation for use in vision-based navigation. The representation has two useful properties: 1) it has constant size, which can enable strong run-time guarantees to be made for control algorithms using it, and 2) it is structurally similar to a camera image space, which effectively allows control to operate in the sensor space rather than employing difficult, and often inaccurate, projections into a structurally different control space (e.g. Euclidean). The presented representation is intended to form the basis of a vision-based subsumption control architecture.

%implementing these frameworks in real systems. \todo{work in constant-space property, mention that representation is in a unified space via sensor/use specific potential transformations}

%Previous work has demonstrated a Selective Determinism framework for decoupling goal-direction and collision avoidance tasks in partially observable multi-agent navigation problems. Under such a decoupling, the implementer loses theoretical optimality guarantees, but in exchange gains theoretical tractability. Crucially, under reasonable assumptions, tractability can be maintained {\em without} losing guarantees on collision avoidance. 
\end{abstract}

%%%%%%%%%%%%%%%%%%%%%%%%%%%%%%%%%%%%%%%%%%%%%%%%%%%%%%%%%%%%%%%%%%%%%%%%%%%%%%%%
\section{INTRODUCTION}
Agents often encounter large and varying numbers of entities within a scene while performing navigation, which can be problematic for conventional planning and control approaches that reason explicitly over all entities. Consider, for instance, busy roadways or crowded sidewalks or convention halls: conventional planning approaches in these scenarios can suffer significant performance degradation as entity count increases~\cite{DBLP:conf/itsc/BrechtelGD11,DBLP:journals/arobots/GalceranCEO17,DBLP:series/sbis/OliehoekA16}. While more scalable approaches exist, they often have strict requirements on system dynamics~\cite{DBLP:conf/isrr/BergGLM09} or observability of agent policies~\cite{DBLP:journals/ijrr/BekrisGMK12}.

To help address the scalability problem, this report builds on \cite{JKJohnson_PPNIV17} to present a constant size environment representation for vision-based navigation. The representation is modeled after a camera image space, which is chosen because cameras are a ubiquitous sensor modality, and image space is typically discretized and constant size. The proposed representation allows planning and control routines to reason almost directly in sensor space thereby avoiding often complex and noisy transformations to and from a more conventional Euclidean space representation. This new representation can help vision-based mobile robots navigate complex multi-agent systems efficiently, and can aid satisfying the strict resource requirements often present in real-time, safety critical, and embedded systems~\cite{DBLP:conf/sigopsE/LouiseDDA02}. Further, the representation enables additional guidance information to be added in while making guarantees about the preservation of hard constraint (Definition~\ref{def:hard-constraint}) information. This property makes it ideal for use in vision-based subsumption control architectures.

%The next section briefly surveys related work, then the environment representation is presented. Finally, conclusions and future work are discussed.

\begin{figure}
  \begin{subfigure}{\linewidth}
  \centering
  \includegraphics[width=0.95\linewidth]{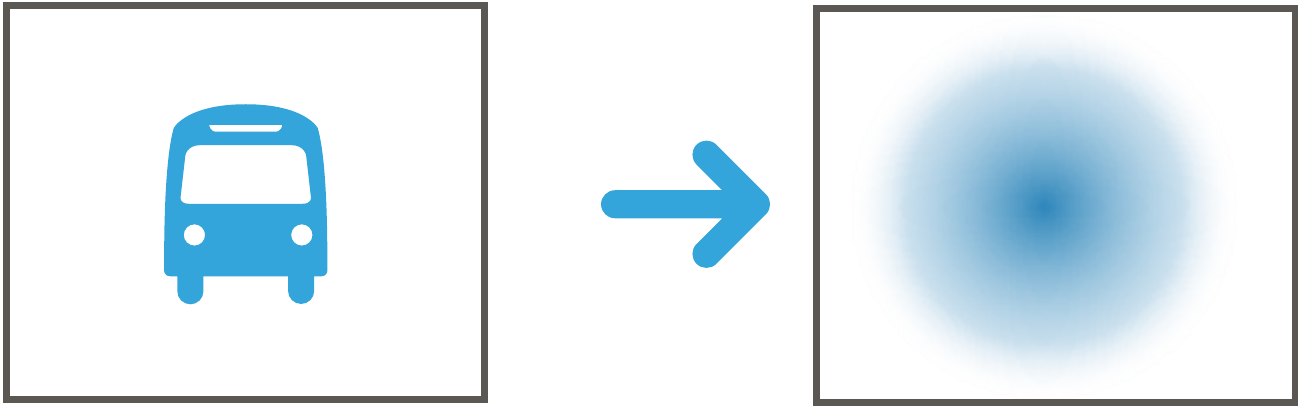}
  \end{subfigure}\par\medskip
  \begin{subfigure}{\linewidth}
  \centering
  \includegraphics[width=0.95\linewidth]{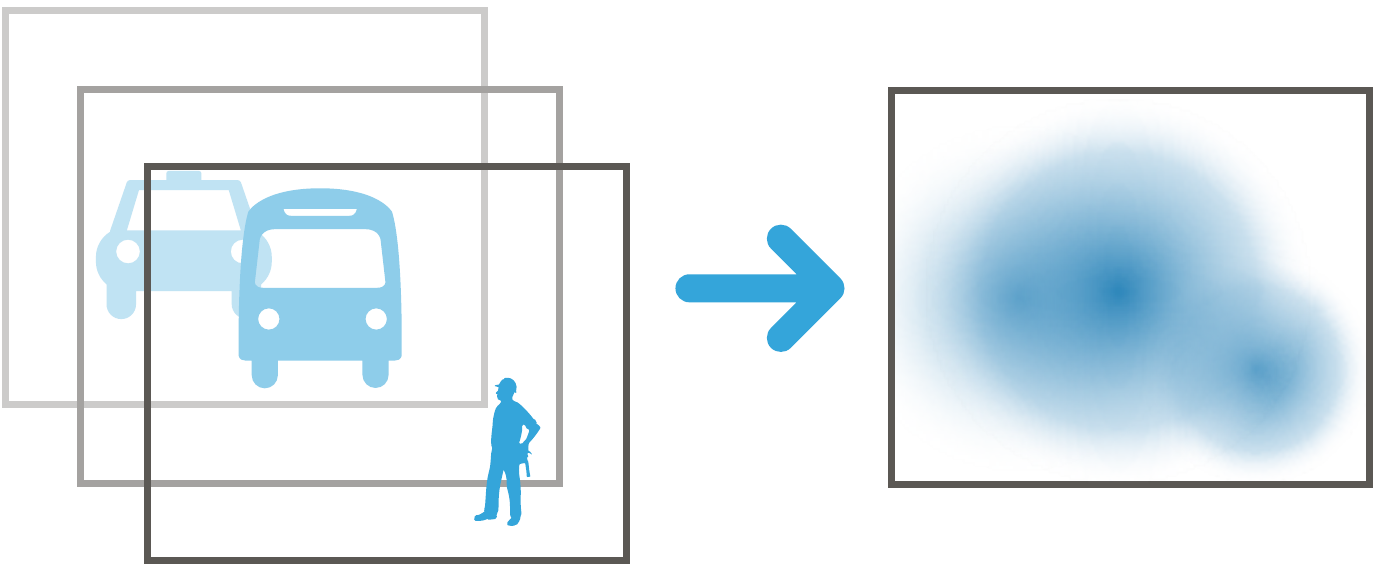}
  \end{subfigure}
\caption{Top: Illustration of an object in the image plane (left), and its Image Space Potential field (right). Bottom: Information from multiple sources (left) can be combined into a single field (right). Black boxes represent ROIs.}\label{fig:cisp-field}
\end{figure}

\section{BACKGROUND}
The approach in this report is based in part on {\em potential fields}~\cite{DBLP:conf/icra/Khatib85}. These fields represent attractive and repulsive forces as scalar fields over a robot's environment that, at any point, define a force acting on the robot that can be interpreted as a control command. As noted in literature, this type of approach is subject to local minima and the {\em narrow corridor} problem~\cite{Koren91potentialfield}, particularly in complex, higher-dimensional spaces~\cite{DBLP:books/daglib/0016830}. Randomized approaches can partially overcome these difficulties~\cite{126256,DBLP:journals/ijrr/BarraquandL91}, while extensions to globally defined {\em navigation functions}~\cite{koditschek-aam-1990,1249705,DBLP:conf/icra/Koditschek87} can theoretically solve them but are often difficult to use in practice. This report uses low-dimensional potential fields, limiting the possibility of narrow corridors, and designs the fields such that additional information can be added in to help break out of minima~\cite{DBLP:conf/icra/Masoud05}. Because the potential fields are modeled after an image space, controlling on them can be accomplished effectively through {\em visual servoing} techniques~\cite{Cowan:CSD-02-1200,DBLP:conf/icra/WeissSN85,DBLP:conf/icra/FomenaC07,DBLP:journals/corr/LeeLA17}.%This is a class of controllers that computes control commands directly from camera data.

In order to define values for the potential fields, this approach draws on a wealth of related works in optical flow and monocular collision avoidance, notably~\cite{DBLP:conf/icra/ByrneT09,DBLP:conf/icra/MoriS13,DBLP:conf/icpr/CamusCHH96,DBLP:conf/iccv/CoombsHHN95,DBLP:conf/scia/Farneback03,DBLP:journals/pami/NelsonA89,NASA-TM-104025}. The intuition of these approaches is that a sequence of monocular images provides sufficient information to compute {\em time-to-contact} (Definition~\ref{def:ttc}), which informs an agent about the rate of change of proximity.

The fields in this representation are intended for use in a subsumption architecture~\cite{DBLP:journals/trob/Brooks86} where additional information about the system can be layered in while hard constraint information is {\em guaranteed} to be preserved. This closure property is proven to hold under a restricted input space with specially constructed potential transform functions.

\section{DEFINITIONS}

%This section gives relevant definitions used in this report. From this section forward, all references to potential fields will refer to one or both of Definitions~\ref{def:potential-field} \&~\ref{def:ae-potential-field}.

\begin{defn}\label{def:potential-field}
A {\em potential field} (also {\em artificial potential field}) is field of artificial forces that attracts toward desired locations and repels from undesirable locations.
\end{defn}

\begin{defn}\label{def:ae-potential-field}
An {\em affinely extended potential field} is a potential field with a potential function that ranges over the affinely extended reals $\overline{\mathbb{R}}=\mathbb{R}\cup\{-\infty,+\infty\}$. A {\em positive} (or {\em negative}) affinely extended potential field is defined over $\overline{\mathbb{R}}$ but contains only positive (or only negative) infinite values.\footnote{For technical reasons, the potential fields are intentionally {\em not} defined over the projectively extended reals $\widehat{\mathbb{R}}=\mathbb{R}\cup\{\infty\}$.}% See Appendix for details.}
\end{defn}

\begin{defn}
An {\em image space potential function} is a mapping of an image pixel value $I(x,y)$ to a tuple in $\overline{\mathbb{R}}^2$ that consists of the potential value and its time derivative:
\begin{align}
I(x,y)\mapsto\overline{\mathbb{R}}^2
\end{align}
\end{defn}

\begin{defn}\label{def:asymptotic-region}
An {\em asymptotic region} $R_A$ is a closed set of points on $\mathbb{R}^2$ such that the potential function takes a value $\pm\infty$ for any element of $R_A$. (This is related to the notion of a {\em natural boundary} in complex analysis.)%
%and the function is {\em lacunary}.
%Note that only affinely extended potential fields can have asymptotic regions.
\end{defn}

\begin{defn}\label{def:encroachment}
{\em Encroachment} is the reduction in minimum proximity between two or more objects in a workspace $\mathcal{W}$ as measured by a metric $\mu(\cdot,\cdot)$.
\end{defn}

\begin{defn}\label{def:gca}
{\em Guided collision avoidance} describes the strategy of choosing goal-directed motions from the space of collision avoiding controls in order to navigate while satisfying collision constraints.%The corresponding problem is that of computing a control sequence according to the guided collision avoidance strategy.
\end{defn}

\begin{defn}\label{def:gn}
{\em Navigation} or {\em multi-agent navigation} describes the general process of navigating a space possibly shared with multiple other agents.%This is a generalization of multi-agent navigation that simply allows for zero or more other agents.%The corresponding problem is that of computing a control sequence that performs such navigation.
\end{defn}

\begin{defn}\label{def:ttc}
{\em Time-to-contact} ($\tau$), is the predicted duration of time remaining before an object observed by a camera will come into contact with the image plane of the camera. The time derivative of $\tau$ is written $\dot{\tau}$.
\end{defn}

\begin{defn}\label{def:hard-constraint}
A {\em hard constraint} is a system constraint that an agent is never allowed to violate.
\end{defn}

\begin{defn}\label{def:soft-constraint}
A {\em soft constraint} is a system constraint that an agent only prefers not to violate.
\end{defn}

\section{THE IMAGE SPACE POTENTIAL FIELD}
Image Space Potential (ISP) fields are affinely extended potential fields that are modeled after image planes. As with image planes, these potential fields can be discretized, and regions of interest (ROIs) can be defined for them. In this report it is assumed that the fields are origin and axis aligned with the agent's camera image plane, and that they have the same ROIs (Figure~\ref{fig:cisp-field}).%An ISP field is allowed to have positive or negative affinely extended potential fields, but not both. The following section will clarify why this is.

\subsection{Representing Hard \& Soft Constraints}
%In mobile robotics applications a distinction is typically made between {\em hard constraints} and {\em soft constraints}. Hard constraints are those that must never be violated, whereas soft constraints are biases that may be violated.
In potential field representations the distinction between hard and soft constraints can be made in terms of the limiting value of the field as the robot approaches some state that would cause constraint violation: the limiting value of the field over states where hard constraint violation would occur can be infinite, such that no reward can overwhelm the cost, and the limiting value of the field over states where soft constraint violation would occur can be finite, such that a reward must be at least some value before the robot chooses to violate it. 

In order for the ISP field representation to be useful it must incorporate this notion of constraints, and it must maintain that notion through summation operations. It is shown below that sums of ISP fields do maintain this information through the use of asymptotic regions, and that they behave as expected so long as the following requirements are met:

\begin{enumerate}
\item All ISP fields involved in summation must have like affine extensions, i.e., all fields must either be positively or negatively affinely extended
\item ISP fields may only be multiplied by scalars in $(0, +\infty)$
\item ISP fields may only be elementwise multiplied by scalar fields where all values are in $(0, +\infty)$
\end{enumerate}

These properties ensure that operations on ISP fields are closed and that asymptotic regions are preserved. The following lemmas prove this:

\begin{lemma}\label{lem:ar-preservation}
Let $F_1$ and $F_2$ be ISP fields, and let $R_A$ be an asymptotic region in $F_1$. Define element $a\in F_2$ as any arbitrary point and define scalar value $s\in(0,+\infty)$. In any field $F_3=F_1+F_2$ or $F_4=s\cdot F_1$, $R_A$ will persist as an asymptotic region.
\end{lemma}
\begin{proof}
It suffices to show that the property holds for a single element of the field. Let $\pm\infty$ be the value of any element of $F_1$ from $R_A$. Then, by definition of addition and multiplication on the affinely extended real number line:

\begin{align*}
a\pm\infty&=\pm\infty,&a\ne\mp\infty\\
s\cdot(\pm\infty)&=\pm\infty,&s\in(0,+\infty)
\end{align*}

The restrictions that ISP fields have like affine extensions and that scalars belong to $(0,+\infty)$ guarantee the conditions in the right column. Thus, any point in $R_A$ with infinite potential in $F_1$ will have infinite potential in $F_3$ or $F_4$.% under addition with $F_2$ or multiplication with $s$.
\end{proof}

\begin{lemma}\label{lem:closure}
Let $F_1$ and $F_2$ be ISP fields. Define elements $a_1\in F_1$ and $a_2\in F_2$ as any arbitrary points in $F_1$ and $F_2$, and define scalar value $s\in(0,+\infty)$. For all $a_3\in F_3=F_1+F_2$ and $a_4\in F_4=s\cdot F_1$, it will be that $a_3,a_4\in\overline{\mathbb{R}}$
\end{lemma}
\begin{proof}
It suffices to show that the property holds for a single element of the field. When $a_1,a_2\in\mathbb{R}$ addition is closed. As noted in Lemma~\ref{lem:ar-preservation} when either, or both, is infinite addition is also closed, assuming $F_1$ are both either positive or negative affinely extended fields. Similar arguments apply to scalar multiplication. However, if $s$ were allowed to range to infinite values, the following would result in an indeterminate form for field value $a$, and closure would be broken:

\begin{align*}
s\cdot a&\notin\overline{\mathbb{R}},&a=0, s=+\infty
\end{align*}

Thus, addition over strictly positive or strictly negative affinely extended fields and element-wise scalar multiplication with $s\in(0,+\infty)$ are both closed operations.

\end{proof}

Note that Lemmas~\ref{lem:ar-preservation} \&~\ref{lem:closure} do {\bf not} hold if ISP fields are allowed infinite values of mixed signs. This is why ISP fields are restricted to only positive or only negative affinely extended potential fields.
%For more details on why this representation was chosen, see the Appendix.

\subsection{Subsumption through Addition}
The results of Lemmas~\ref{lem:ar-preservation} \&~\ref{lem:closure} are powerful because they imply that the information of arbitrary fields can be added together without losing information about hard constraints. Thus, a control architecture using ISP fields can implement subsumption through addition.
% This enables a simple and elegant method for fusing information from multiple sources with the guarantee that hard constraint information will not be lost: simply add the fields together, and it is in this way that ISP fields form a Subsumption Architecture.

\section{The Potential Function}

The potential function, which maps image pixel values to potential values, can be defined in arbitrary ways, either with geometric relations, learned relations, or even heuristic methods. In this report, hard constraint information will be derived from a geometric measure over pixels in a temporal image sequences called {\em time-to-contact}, or $\tau$. Soft constraint information will be represented by user-specified values meant to bias how strongly directed a chosen control is to a particular goal. The potential function maps these measurement values to a unitless potential value space through the {\em potential transformation} defined later in this section.

\subsection{Obtaining Hard Constraint Values}\label{sec:computing-tau}
As noted often in literature (e.g.~\cite{NASA-TM-104025,DBLP:conf/icpr/CamusCHH96,DBLP:conf/iccv/CoombsHHN95,DBLP:journals/pami/NelsonA89}), $\tau$ can be computed directly from the motion flow of a scene, which is the vector field of motion in an image due to relative motions between the scene and the camera. Unfortunately, it is typically not possible to measure motion flow directly, so it is usually estimated via {\em optical flow}, which is defined as the {\em apparent} motion flow in an image plane. Historically this has been measured by performing some kind of matching of, or minimization of differences between, pixel intensity values in subsequent image frames~\cite{DBLP:journals/ai/HornS81,DBLP:conf/scia/Farneback03,DBLP:conf/ijcai/LucasK81}. More recently deep learning techniques have also been successfully applied to the problem~\cite{DBLP:conf/iccv/WeinzaepfelRHS13}.

Assuming some reasonably accurate estimation of optical flow vector field exists, $\tau$ can be computed directly under certain assumptions~\cite{DBLP:conf/icpr/CamusCHH96}. A significant advantage of computing $\tau$ from optical flow, in particular dense optical flow, is that the computation is independent of the number of objects in a scene. In other words, given optical flow, the scalability issues of object tracking and segmentation can be avoided. In practice, however, the computation of optical flow can be noisy and error prone, so feature- and segmentation-based approaches can also be used~\cite{DBLP:conf/icra/MoriS13,DBLP:conf/icra/ByrneT09}. The idea of these approaches is to compute $\tau$ from the rate of change in detection {\em scale}. For a point in time, let $s$ denote the scale (maximum extent) of an object in the image, and let $\dot{s}$ be its time derivative. When the observed face of the object is roughly parallel to the image plane, and under the assumption of constant velocity translational motion and zero yaw or pitch, it is straightforward to show that~\cite{camus-phdthesis}:

\begin{align}
\tau=\frac{s}{\dot{s}}\label{eq:tau-scale}
\end{align}

As shown in~\cite{JKJohnson_PPNIV17}, scale has a useful invariance property for these types of calculations that can make $\tau$ computations robust to certain types of noise and assumption violations. Lemma~\ref{lem:scale-invariance} demonstrates this:

\begin{lemma}\label{lem:scale-invariance}
The scale $s$ of an object on the image plane is invariant to transformations of the object under $SE(2)$ on the $XY$ plane.
\end{lemma}

\begin{proof}
Let $(X_1,Y_1,Z)$ and $(X_2,Y_2,Z)$ be end points of a line segment on the $XY$ plane in the world space, with $XY$ parallel to the image plane and $Z$ coincident with the camera view axis. Without loss of generality, assume unit focal length. The instantaneous scale $s$ of the line segment in the image plane is given by:

\begin{align}
s=\frac{1}{Z}\sqrt{\Delta X^2+\Delta Y^2}
\end{align}

Thus, any transformation of the line segment on the $XY$ plane for which $\Delta X^2+\Delta Y^2$ is constant makes $s$, and thereby $\dot{s}$ and $\tau$, independent of the values of $(X_1,Y_1)$ and $(X_2,Y_2)$. By definition, $SE(2)$ satisfies this condition.
\end{proof}

In addition, the time derivative $\dot{\tau}$ of $\tau$, when available, enables a convenient decision function for whether an agent's current rate of deceleration is adequate to avoid head-on collision or not~\cite{citeulike:2798975}:

\begin{align}
   \dot{\tau}\mapsto\left\{
     \begin{array}{lr}
       1 & : \dot{\tau}\ge-0.5\\
       0 & : \dot{\tau}<-0.5
     \end{array}
   \right.\label{eq:tau-dot-decision}
\end{align}

Equation~\ref{eq:tau-scale} allows the computation $\tau$ for whole regions of the image plane at once given a time sequence of labeled image segmentations, while $\dot{\tau}$ enables decisions to be made about the safeness of the agent's current state.

\subsection{Obtaining Soft Constraint Values}
Whereas hard constraint values are measurements of geometric properties and intended to prevent undesired physical interactions in the world, soft constraint values are intended only to bias how control actions are chosen. The potential transformation described in the next section guarantees that soft constraint values can never override hard constraints values, which allows soft constraint values to be arbitrarily chosen. In this report, soft constraint values are assume to be user-specified.

The following two sections describe how hard and soft constraint values are transformed into the potential space.

\subsection{The Potential Transformation}
\begin{figure}
  \begin{subfigure}{\linewidth}
  \centering
  \includegraphics[width=1\linewidth]{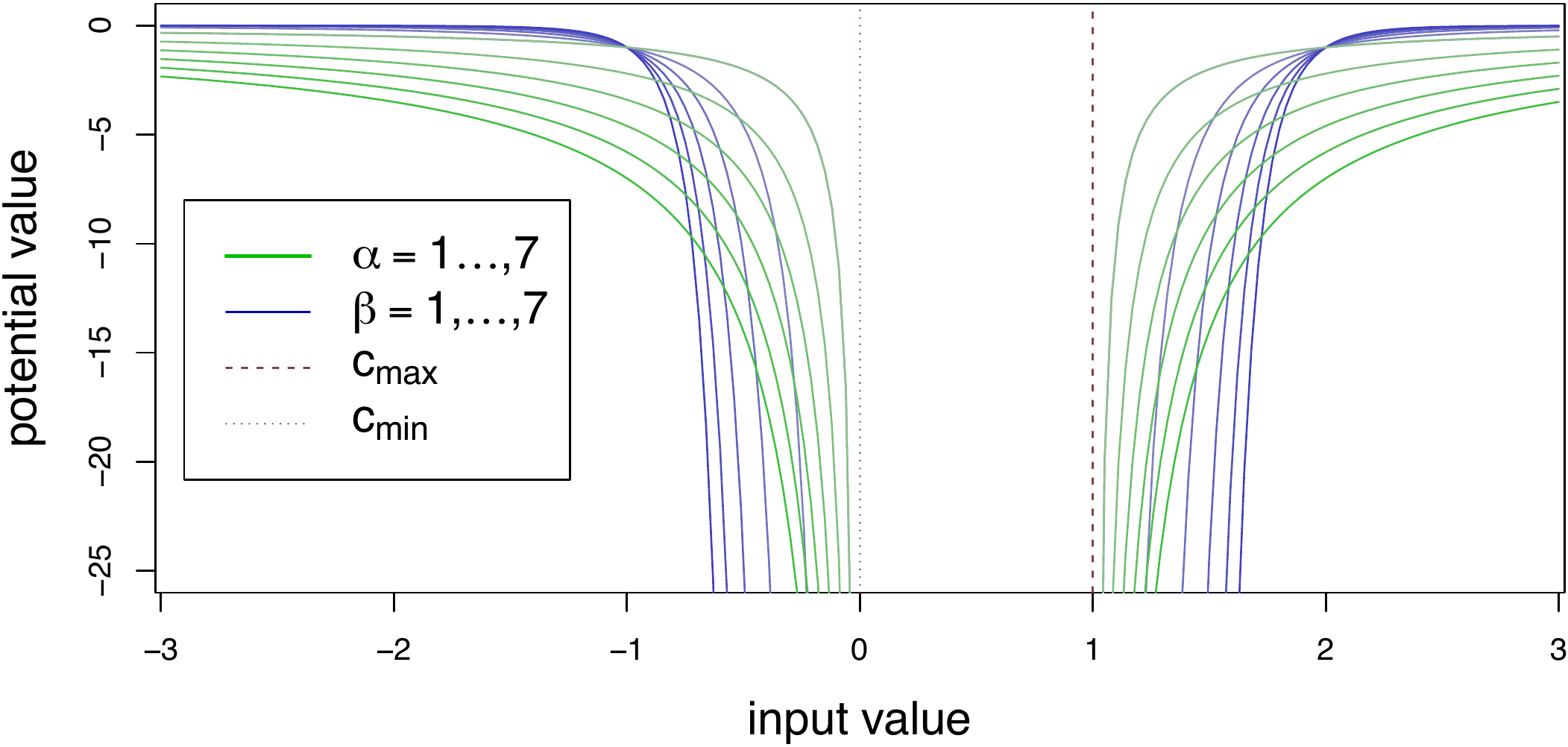}
  \end{subfigure}
\caption{This plot shows a selection of shape parameters settings in Equation~\ref{eq:hard-constraint-mapping}. The green lines show $\alpha$ ranging from $1,\ldots,7$ with $\beta=1$ and the blue lines show $\beta$ ranging from $1,\ldots,7$ with $\alpha=1$. The constraint range is taken to be $\left[\underline{c},\overline{c}\right]=[0,1]$. Translation is $0$. Best viewed in color.}\label{fig:hard-constraint-plot}
\end{figure}

\begin{figure}
  \begin{subfigure}{\linewidth}
  \centering
  \includegraphics[width=1\linewidth]{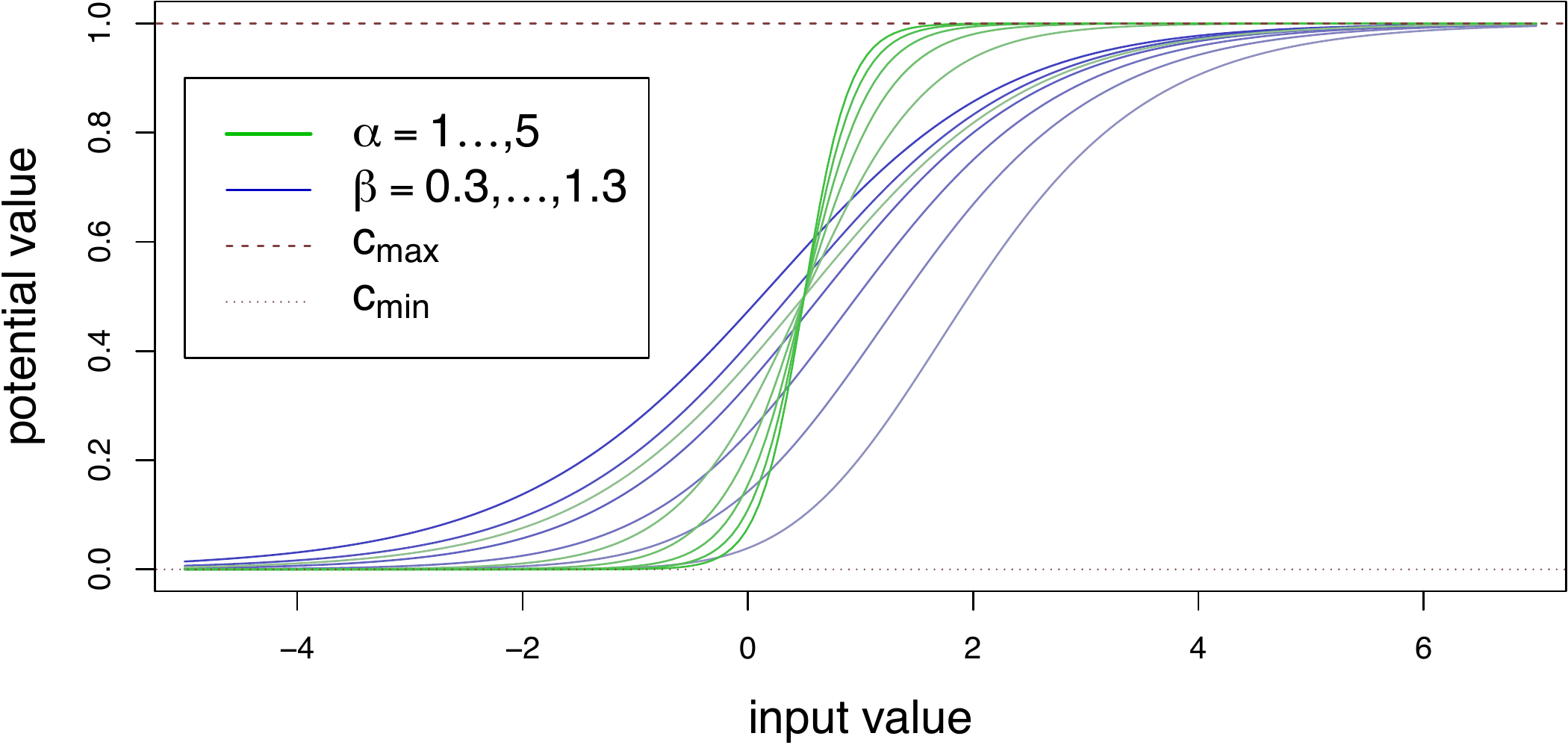}
  \end{subfigure}
\caption{This plot shows a selection of shape parameter settings in Equation~\ref{eq:soft-constraint-mapping}. The green lines show $\alpha$ ranging from $1,\ldots,5$ with $\beta=1$ and the blue lines show $\beta$ ranging from $0.3,\ldots,1.3$ with $\alpha=1$. The constraint range is taken to be $\left[\underline{c},\overline{c}\right]=[0,1]$. Translation is $0$. Best viewed in color.}\label{fig:soft-constraint-plot}
\end{figure}

The task of projecting sensor measurements and user-defined bias values is accomplished by a {\em potential transformation} that transforms pixel-wise measurements that have some semantic meaning into the unitless potential space. Two types of potential transformations are presented: a hard constraint transform and a soft constraint transform. Without loss of generality, assume for the remainder of this report only negative affinely extended potential fields.%Note that the potential transforms presented in this section are designed to be generally useful, they are by no means unique.

The hard constraint transform is intended to map values within a specific range to infinite potential values, and values outside the range to finite potential values. In this way sensor measurements corresponding to hard constraint violations are mapped into asymptotic regions, which ensures that information about them is preserved during subsumption. At time $t$, for a finite input pixel value \pixelval, finite range $\left[\underline{c},\overline{c}\right]$, y-axis translation value $t_y$, and finite shape parameters $\alpha,\beta>0$, the hard constraint transform $C_h$ is given in Equation~\ref{eq:hard-constraint-mapping} and illustrated in Figure~\ref{fig:hard-constraint-plot}. The time derivative of $C_h$ is given in Equation~\ref{eq:hard-constraint-mapping-dot}.

\begin{align}
   C_h\left(\pixelval,\left[\underline{c},\overline{c}\right]\right)=\left\{
     \begin{array}{ll}
     t_y - \alpha(\underline{c}-\pixelval)^{-\beta} & : \pixelval<\underline{c}\\
     -\infty & : \pixelval\in\left[\underline{c},\overline{c}\right]\\
     t_y - \alpha(\pixelval-\overline{c})^{-\beta} & : \pixelval>\overline{c}
     \end{array}
   \right.\label{eq:hard-constraint-mapping}\\
  \frac{d}{dt}C_h=\left\{
     \begin{array}{ll}
     \alpha\beta(\underline{c}-\pixelval)^{-\beta-1}\frac{d}{dt}\pixelval & : \pixelval<\underline{c}\\
     0 & : \pixelval\in\left[\underline{c},\overline{c}\right]\\
     \alpha\beta(\pixelval-\overline{c})^{-\beta-1}\frac{d}{dt}\pixelval & : \pixelval>\overline{c}
     \end{array}
   \right.\label{eq:hard-constraint-mapping-dot}
\end{align}

Conversely, the soft constraint transform is intended to map all values into a given finite range. In this way all user-given bias values are prevented from becoming part of asymptotic regions, which ensures that hard constraint information is not corrupted. At time $t$, for a finite input value \pixelval, finite range $\left[\underline{c},\overline{c}\right]$ with midpoint $c_\textrm{mid}$, x-axis translation value $t_x$, and finite shape parameters $\alpha,\beta>0$, the soft constraint transform $C_s$ is a parameterized logistic function given in Equation~\ref{eq:soft-constraint-mapping} and illustrated in Figure~\ref{fig:soft-constraint-plot}. The time derivative is given in Equation~\ref{eq:soft-constraint-mapping-dot}.

\begin{align}
   C_s\left(\pixelval,\left[\underline{c},\overline{c}\right]\right)=\underline{c} + \frac{\overline{c}-\underline{c}}{\left(1+ e^{-\alpha\left(\pixelval-t_x-c_\textrm{mid}\right)}\right)^{1/\beta}}\label{eq:soft-constraint-mapping}\\
   \frac{d}{dt}C_s=\frac{\alpha\left(\overline{c}-\underline{c}\right)e^{-\alpha\left(\pixelval-t_x-c_\textrm{mid}\right)}}{\beta\left(1+e^{-\alpha\left(\pixelval-t_x-c_\textrm{mid}\right)}\right)^{1+1/\beta}}\frac{d}{dt}I^t_{mn}\label{eq:soft-constraint-mapping-dot}
\end{align}

The shape parameters in both Equations~\ref{eq:hard-constraint-mapping} \&~\ref{eq:soft-constraint-mapping} allow potential field shape to be manipulated in problem-specific ways. The time derivatives in Equations~\ref{eq:hard-constraint-mapping-dot} \&~\ref{eq:soft-constraint-mapping-dot} are useful in the situation that a time derivative of the input pixel value is given. For convenience, let $\Gamma=\left\langle t,\underline{c},\overline{c},\alpha,\beta\right\rangle$ contain all parameters for computing a constraint transform.%Also note that these transform functions are not unique; for certain problems different formulations may be more useful.

\section{VISION-BASED SUBSUMPTION ARCHITECTURE}
This section outlines a subsumption-based control architecture for ISP fields. Subsumption is implemented by formulating the general navigation problem as a guided collision avoidance problem, in which a global guidance controller subsumes a local collision avoidance controller.%This architecture addresses a general navigation problem through guided collision avoidance.

%\subsection{Problem Decomposition}
The navigation problem can be described as below:

\begin{problem}\label{prob:main-problem} {\em Navigation}: Let \set{A} be a set of agents navigating along a 2D manifold and assume that collision is never inevitable in the initial system state. Assume each agent is equipped with cameras, and that each agent has knowledge of the physical dynamic properties of the environment and of other agents. Assume each agent actuates according to a unique decision process and that each agent may assume with certainty that other agents will prefer to avoid collision and to avoid causing collision. How can an agent $A\in\set{A}$ navigate toward a goal while remaining collision free?
\end{problem}

Problem~\ref{prob:main-problem} can be decomposed into a local collision avoidance problem, and a global guidance problem. The local problem is given below:

\begin{problem}\label{prob:encroachment-detection} {\em Collision Avoidance}: Assume an agent $A$ navigating a workspace \set{W} receives some observation input $O_t$ of \set{W} over time. Let $\set{A}$ be the set of objects and agents that does not include $A$. For a distance metric $\mu$, threshold $\varepsilon>0$, and a sequence of observations $O_i,\ldots,O_t$, how can $A$ estimate the rate of change of $\min_{A_i\in\set{A}}\mu(A,A_i)$ such that it can compute controls to maintain $\min\mu>\varepsilon$?
\end{problem}

The metric $\mu$ between any two points is taken as the $\tau$ measure between those points, thus, Problem~\ref{prob:encroachment-detection} makes the assumption that maintaining some $\tau>\varepsilon$ suffices to ensure collision avoidance, which is often a reasonable assumption~\cite{lamm1999highway}. This assumption implies several other assumptions, namely, that agents maintain controllability at all times, and that agents react within the $\tau$ horizon. These assumptions can be loosened in a probabilistically rigorous way by exploiting the Safety-Constrained Interference Minimization Principle~\cite{Johnson_semiautonomouslongitudinal,6338731}.
%For clarity of presentation of this report, algorithms and methodology for doing so is reserved for future work.

Finally, the global guidance controller problem can be defined as:

\begin{problem}\label{prob:guidance-control} {\em Guidance:} For desired goal-directed control $u^d$, control space metric $\mu_c$, and given a feasible control set $\mathcal{U}$, choose a control $u^\star$ such that:
\begin{align*}
u^\star=\arg\min_{u\in\mathcal{U}}\mu_c(u^d,u)
\end{align*}
\end{problem}

Problem~\ref{prob:guidance-control} can be solved any number of optimization techniques. In particular, if the solution algorithm to Problem~\ref{prob:encroachment-detection} produces convex sets, many efficient optimization routines become available.

%In this architecture, subsumption occurs when the control set \set{U} is generated using the local collision avoidance controller.

\subsection{Example Algorithms}
Algorithms~\ref{alg:virtual-bumpers} \&~\ref{alg:virtual-bumpers-sd} address Problems~\ref{prob:main-problem} \&~\ref{prob:encroachment-detection} explicitly with Problem~\ref{prob:guidance-control} being straightforward to solve given the input. This solution is intended as a sketch of a control system based on ISP fields, so some details are omitted. The biasing fields in Algorithm~\ref{alg:virtual-bumpers-sd} are derived from user-provided values.%A more thorough treatment of the control algorithm may be given in future work, or the reader may examine the implementation available in~\cite{maeve-development-libraries}.
%Problem specific sophistication can be added simply by computing additional ISP fields and accumulating them together to manipulate agent behavior.

While Problem~\ref{prob:main-problem} is formulated generally similar to a ground navigation scenario, ISP fields can applied to arbitrary environs so long as Algorithms~\ref{alg:virtual-bumpers} \&~\ref{alg:virtual-bumpers-sd} are modified accordingly.

\section{Conclusion}
This report presented Image Space Potential (ISP) fields, which are a general environment representation for vision-based navigation. ISP fields are constant space complexity with respect to the image, which is crucial for ensuring scalability and running time of algorithms. ISP fields also enable planning and control to occur in a space that is structurally similar to the sensor space, which means that sensor data does not need to go through what are often difficult and noisy projections into a structurally different planning and control space.

%An experimental demonstration of a novel control law using the ISP field representation was given on a physical test platform. The experiments demonstrated that sophisticated navigation behavior can be achieved in dynamic environments without the need for sophisticated models of the robot's actuation system and without the need for accurate predictions of how the environment evolves over time.

ISP fields are intended to form the foundation of a vision-based subsumption control architecture, such as that described in the previous section. To enable this use, they allow arbitrary amounts of guidance information to be added into the representation while guaranteeing that hard constraint information will not be lost or corrupted in the process. This  capability makes ISP field representation particularly well suited to enabling machine learning to be applied to safety critical applications, such as automated driving, where it is often impractical for machine learning alone to make such guarantees~\cite{DBLP:journals/corr/Shalev-ShwartzS16a}.

An implementation of the data structures and algorithms described in this report is being developed and maintained under open source license~\cite{maeve-development-libraries}. The implementation is expected to change and grow over time, so for any disparity between the implementation and this document, the implementation should be assumed to be authoritative.

The implementation itself is developed under ROS~\cite{ROS}, and ISP field data structures and operations are implemented using OpenCV~\cite{opencv_library}, a highly optimized, industry standard computer vision library.

\begin{algorithm}
\caption{This algorithm addresses Problem~\ref{prob:encroachment-detection}. Given an ISP field $F$, compute the set of steering and acceleration commands where $w_\theta$ is a kernel for computing the steering angle to acceleration map at a single horizon line $h$, $K_P$, $K_D$ are proportional and derivative gains, and $\Gamma_u=\langle t, \underline{c}, \overline{c}, \alpha_u, \beta_u\rangle$ parameterizes a soft constraint mapping onto the set of normalized valid longitudinal controls.}
\label{alg:virtual-bumpers}
\begin{algorithmic}[1]
\Procedure{ControlSet}{$F,\varepsilon,w_\theta, h, K_P, K_D$}
\State Let $I_c$ be the list of image column indices
\State Let $M_\theta$ map $\langle p,\dot{p}\rangle$ to $I_c$ via $w_\theta$ erosion along $h$\label{line:theta-window}
\State Let $M_u\gets\emptyset$ map $i\in I_c$ onto sets of accelerations
\For{$i\in I_c$}
\State $\langle p, \dot{p}\rangle\gets M_\theta[i]$
\State $u\gets\langle p, \dot{p}\rangle\cdot\langle K_P, K_D\rangle$
\State $a_\textrm{max}\gets C_s(u, \Gamma_u)$
\State $M_u[i]\gets\left[\underline{c}, a_\textrm{max}\right]$
\EndFor
\State\Return$M_u$s
\EndProcedure
\end{algorithmic}
\end{algorithm}

\begin{algorithm}
\caption{This algorithm addresses Problem~\ref{prob:main-problem}. For a given ISP field $F$ and guidance control $u^d=\langle\theta^d,a^d\rangle$ from Problem~\ref{prob:guidance-control}, compute the control that safely controls the agent $A$. See Algorithm~\ref{alg:virtual-bumpers} for descriptions of the other parameters.}
\label{alg:virtual-bumpers-sd}
\begin{algorithmic}[1]
\Procedure{SDControl}{$u^d,F,\varepsilon,w_\theta,h,K_P,K_D$}
\State Let $s$ be the current state of $A$
\State Let $I_c$ be the list of image column indices
%\State Let $\theta_d$ be the steering angle corresponding to $y_d$
\State Let $M_u\gets\textrm{ControlSet}(F,\varepsilon,w_\theta,h,K_P,K_D)$\label{line:safe-control}
\State Let $F_B^d\gets $ biasing field from $u^d$
\State Let $F_B^u\gets $ biasing field from $M_u$
\State Let $F^\star\gets F + F_B ^d + F_B^u$
\State Let $H\gets $ min reduce $F^\star$ along $h$
\State Let $i^\star\gets $ max-valued column index of $H$
\State $\theta^\star\gets$ steering angle corresponding to $i^\star$
\State $a^\star\gets\arg\min_{a\in M_u[i]}\mu_c(a,a^d)$
\State\Return$\langle\theta^\star,a^\star\rangle$
\EndProcedure
\end{algorithmic}
\end{algorithm}

\bibliographystyle{IEEEtran}
\bibliography{IEEEabrv,refs,aaai2018}

\end{document}